\documentclass[letterpaper]{article} 
\usepackage{aaai2026}  
\usepackage{times}  
\usepackage{helvet}  
\usepackage{courier}  
\usepackage[hyphens]{url}  
\usepackage{graphicx} 
\usepackage{natbib}  
\usepackage{caption} 
\usepackage{xr}
\usepackage{amsmath,amssymb,mathtools}
\usepackage[capitalise]{cleveref}

\usepackage{algorithm}
\usepackage{algpseudocode}
\usepackage{array}
\usepackage{amsfonts} 
\usepackage{mathtools}

\DeclareMathOperator*{\argmin}{argmin}
\usepackage{multirow}
\usepackage{multicol}
\usepackage{graphicx}
\usepackage{float}
\usepackage{lipsum,booktabs}
\usepackage{tikz}
\usetikzlibrary{bayesnet}
\usepackage{subcaption}
\newtheorem{theorem}{Theorem}
\newtheorem{assumption}{Assumption}

\newcommand{\E}{\mathbb{E}}

\newcommand{\KL}{\mathrm{KL}}

\newenvironment{proof}{\par\noindent\textbf{Proof.}\ }{\hfill$\square$\par}
\newtheorem{remark}{Remark}
\newtheorem{proposition}{Proposition}

\usepackage{newfloat}
\usepackage{listings}
\DeclareCaptionStyle{ruled}{labelfont=normalfont,labelsep=colon,strut=off} 
\floatstyle{ruled}
\newfloat{listing}{tb}{lst}{}
\floatname{listing}{Listing}
\title{Fair Bayesian Data Selection via Generalized Discrepancy Measures}
\author{
    Yixuan Zhang\textsuperscript{\rm 1}\thanks{These authors contributed equally to this work.},
    Jiabin Luo\textsuperscript{\rm 2}\footnotemark[1],
    Zhenggang Wang\textsuperscript{\rm 1},
    Feng Zhou\textsuperscript{\rm 3,4}\thanks{Corresponding author.},
    Quyu Kong\textsuperscript{\rm 5}
}
\affiliations{
    \textsuperscript{\rm 1}School of Statistics and Data Science, Southeast University, China\\
    \textsuperscript{\rm 2}School of Software and Microelectronics, Peking University, China\\
    \textsuperscript{\rm 3}Center for Applied Statistics and School of Statistics, Renmin University of China, China\\
    \textsuperscript{\rm 4}Beijing Advanced Innovation Center for Future Blockchain and Privacy Computing, China\\
    \textsuperscript{\rm 5}Alibaba Cloud, China
    
    zh1xuan@hotmail.com, jbluo25@stu.pku.edu.cn, zgwangsrmkph@gmail.com, feng.zhou@ruc.edu.cn, kongquyu.kqy@alibaba-inc.com
}

\usepackage{bibentry}

\begin{document}

\maketitle

\begin{abstract}
Fairness concerns are increasingly critical as machine learning models are deployed in high-stakes applications. While existing fairness-aware methods typically intervene at the model level, they often suffer from high computational costs, limited scalability, and poor generalization. To address these challenges, we propose a Bayesian data selection framework that ensures fairness by aligning group-specific posterior distributions of model parameters and sample weights with a shared central distribution. Our framework supports flexible alignment via various distributional discrepancy measures, including Wasserstein distance, maximum mean discrepancy, and $f$-divergence, allowing geometry-aware control without imposing explicit fairness constraints. This data-centric approach mitigates group-specific biases in training data and improves fairness in downstream tasks, with theoretical guarantees. Experiments on benchmark datasets show that our method consistently outperforms existing data selection and model-based fairness methods in both fairness and accuracy.
\end{abstract}


\section{Introduction}
\label{sec:intro}

Artificial intelligence is rapidly expanding into key areas such as clinical diagnosis~\citep{diagnosis}, text generation~\citep{text_generation}, and financial credit approval~\citep{credit_example}. While these advanced models are powerful, they often exhibit uneven performance across different groups, such as those defined by gender, race, or socioeconomic status, which leads to unfair decisions and raising concerns about fairness risks in real-world applications~\citep{bird2016exploring}. As a result, ensuring fairness and preventing AI from exacerbating social inequalities have become critical challenges for both researchers and industry. Fairness-aware machine learning has thus emerged as a key area of study to address these issues.

Currently, most fairness-aware machine learning strategies focus on modifying the model itself~\citep{Zafar2015LearningFC,2018_icml_reductions,baharlouei2024fferm,chen2024post}. However, these approaches often require building new models from scratch for each task, resulting in high computational costs and low efficiency, which makes them difficult to scale to large datasets and deep neural networks. Moreover, even trained fair models may still encounter unfairness issues during transfer due to generalization errors~\citep{dutt2024fairtune}.        

\begin{figure}
    \centering
    \includegraphics[width=0.95\linewidth]{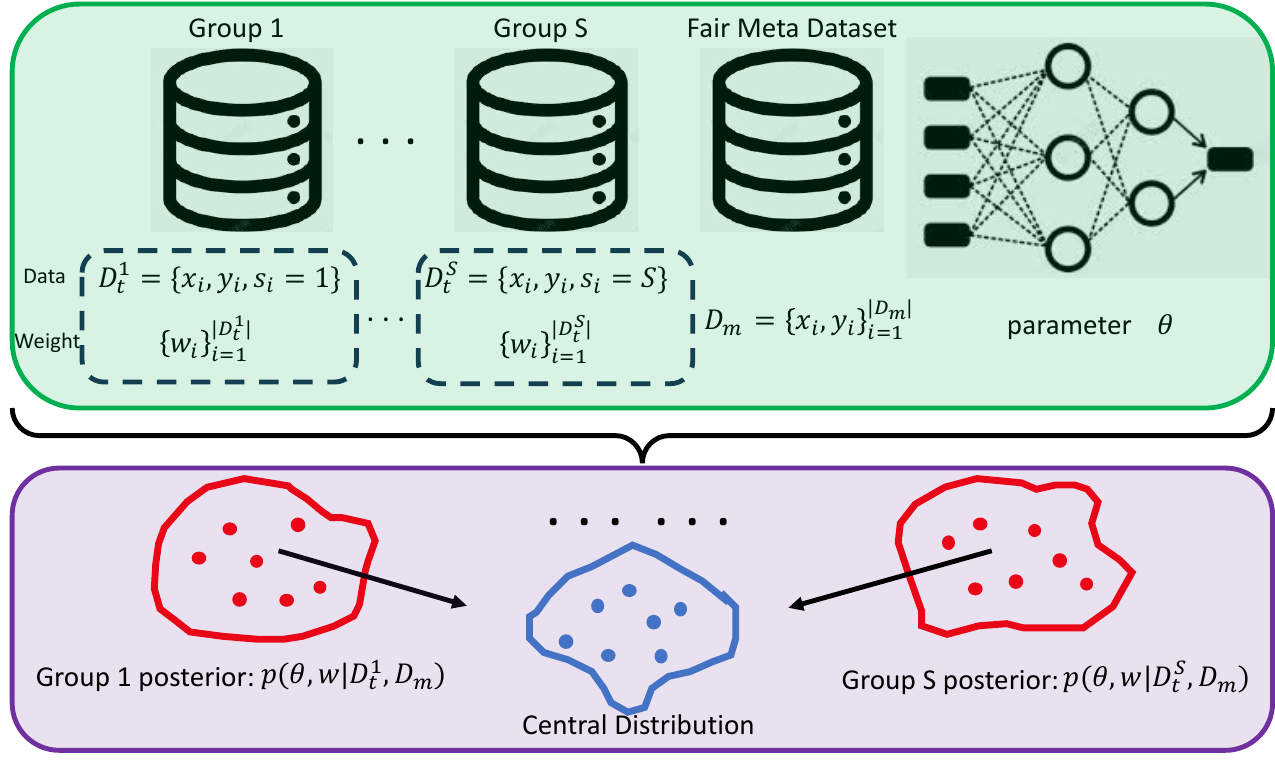}
    \caption{An illustration of Fair-BADS. Fair-BADS jointly infers model parameters and sample weights while reducing bias via posterior alignment to a central distribution.}
    \label{fig:scheme}
\end{figure}

Rather than addressing fairness solely during model training or transfer, a key challenge is to tackle the problem directly from the data. This is increasingly important as modern models rely heavily on massive raw data, which often contain imbalance and systemic biases. Meanwhile, manually identifying high-quality data from such large-scale sources is impractical. As a result, selecting high-quality and fairness-aware training data has become critical for building both effective and fair models~\citep{xu2024a}.

Data selection offers a practical solution for improving model fairness by identifying or reweighting training examples that encourages fair outcomes. However, most existing methods focus on maximizing data utility or informativeness, often overlooking fairness and inadvertently reinforcing biases against underrepresented groups. Moreover, many rely on bi-level optimization or meta-learning, which are computationally intensive and difficult to scale~\citep{fan2017learningdatalearn}. To address these limitations, Bayesian data selection~\citep{xu2024a} formulates the task as posterior inference over model parameters and sample weights, using stochastic gradient Langevin dynamics (SGLD) for efficient optimization. This avoids nested optimization, enables standard gradient-based updates, and therefore scales well to large models and datasets. As manually curating high-quality data becomes infeasible at scale, principled data selection has become essential for reducing noise, imbalance, and bias, making it a critical component of scalable, fairness-aware learning.

Building on this motivation, we propose a fairness-aware data selection framework. While existing approaches often attempt to enforce fairness by adjusting model parameters in Euclidean space, this overlooks the fact that model parameters naturally reside in a more complex, generally non-Euclidean space. To better reflect fairness in this setting, it is crucial to consider the intrinsic geometry of the parameter space and the distribution of group-specific posteriors within it. We propose a Bayesian framework that formulates fairness as the alignment of posterior distributions across demographic groups toward a shared central distribution under a general class of divergence based objectives. Specifically, we jointly infer model parameters and sample weights using a fairness-aware meta-dataset, encouraging group-specific posteriors to align toward this central distribution via divergences such as Wasserstein distance, maximum mean discrepancy (MMD), or $f$-divergence. To efficiently approximate high-dimensional posteriors, we adopt Stein variational gradient descent (SVGD), which deterministically updates particles while preserving diversity. This enables stable and scalable inference without injecting noise and prevents dominant-group bias from overwhelming the learned posterior, making the framework well-suited for fairness-aware learning in large-scale settings. 

Our contributions can be summarized as follows: (1) We propose a data-centric Bayesian framework for fairness-aware learning that jointly infers model parameters and instance weights, providing a scalable alternative to traditional model-centric approaches. (2) We propose a unified divergence-based formulation that aligns group-specific posteriors toward a shared central distribution, enabling flexible and geometry-aware fairness across demographic groups. (3) We provide theoretical guarantees by deriving discrepancy based bounds that approximate average group risk and bound intergroup performance gaps. (4) We use SVGD for efficient posterior inference, enabling stable updates without nested optimization and ensuring scalability for fairness-aware learning. 


\section{Related Works}
\label{sec:related_work}

\paragraph{Data selection.} 
Most established selection strategies rely on bi-level optimization or meta-learning frameworks~\citep{grangier2023adaptivetrainingdistributionsscalable,pmlr-v80-ren18a,10.5555/3454287.3454459,zhang2021learningfastsamplereweighting}, which introduce an additional outer optimization loop to improve training data by maximizing model performance on a held-out meta set. These approaches, including sample-weighting~\citep{grangier2023adaptivetrainingdistributionsscalable,pmlr-v80-ren18a} and mini-batch reweighting~\citep{fan2017learningdatalearn}, often require expensive meta-gradients or reinforcement learning, making them difficult to scale to large datasets and deep models. Other strategies rely on heuristics such as loss or confidence scores, for instance, curriculum learning~\citep{Curriculum/10.1145/1553374.1553380} favors easy samples, online methods~\citep{online_batch,importance_sampling,accelerating_deep_learning} prioritize high-loss or high-gradient examples, and confidence-based approaches~\citep{LongReMix,confidence_score_idn} select uncertain instances. Most methods focus on performance and neglect fairness, with only a few adjusting sampling to meet fairness metrics~\citep{Roh0WS21_fairbatch}. We address this gap with a Bayesian data selection framework that aligns group-specific posteriors to incorporate fairness directly.

\paragraph{Fairness-aware learning.} 
Existing bias mitigation methods generally fall into three categories: preprocessing, in-processing, and post-processing. Preprocessing methods aim to reduce discriminatory information in the input data through fair representation learning~\citep{2013_iclr_ae,icml_2013,2016_lum,DBLP:flex}, fair data generation~\citep{Jang_Zheng_Wang_2021}, and data mapping~\citep{NIPS2017_optimised_preprocessing}. In-processing methods reduce bias during training by incorporating fairness constraints into the learning process~\citep{pmlr-v119-roh20a,baharlouei2024fferm,donini2020empirical,pmlr-v97-gordaliza19a,conf/aaai/ChiappaJSPJA20,zhang_adversarial}. These can be model-specific~\citep{fair_constraints,5360534} or model-agnostic~\citep{2018_icml_reductions,lowy2022a}. Post-processing methods adjust model outputs to meet fairness criteria~\citep{equal_opportunity}. However, these approaches often face scalability and generalization challenges, motivating a shift toward data selection. \citet{fairness_through_aleatoric} is conceptually related in addressing both data distribution and posterior weight biases while overcoming the SGLD limits, ours instead enforces fairness via Bayesian data selection optimized with SVGD.

\section{Preliminaries}
In this section, we review the framework of Bayesian data selection from a fairness perspective and motivate the need for an efficient approach to posterior inference in this setting. 

\subsection{Bayesian Formulation for Fair Data Selection}
Consider a training dataset $\mathcal{D}_t = \{(\mathbf{x}_i, y_i, s_i)\}_{i=1}^N$, where $\mathbf{x}_i$ denotes the non-sensitive features, $y_i \in \{0,1\}$ is the binary label, and $s_i \in \{0,1,\ldots,S\}$ represents the sensitive attribute, such as gender or race. The training set $\mathcal{D}_t$ may contain biased samples due to label corruption influenced by sensitive attributes. For instance, a qualified individual (i.e., $y_{\text{true}} = 1$) might be assigned a negative label ($y_{\text{obs}} = 0$) due to group-based prejudice, as in $p(y_{\text{obs}} = 0 \mid y_{\text{true}} = 1, s)$. Such biases can significantly degrade a model’s fairness performance, especially when these patterns are learned and amplified during training.

To mitigate this issue, we assume access to a small meta-dataset $\mathcal{D}_m$ drawn from a fair target distribution, where labels are unaffected by sensitive attributes: $p(\mathbf{x}, y, s) = p(y \mid \mathbf{x}) p(\mathbf{x}) p(s)$. Traditional data selection methods typically rely on bi-level optimization or meta-learning, where model parameters are trained on a reweighted training set, and the weights are optimized in an outer loop guided by $\mathcal{D}_m$. However, such approaches often incur high computational overhead and instability due to nested optimization. 

The Bayesian formulation~\citep{xu2024a} offers a principled alternative by introducing a probabilistic model over both model parameters $\boldsymbol\theta \in \mathbb{R}^P$ and instance-level sample weights $\mathbf{w} \in \mathbb{R}^N$ applied to the training data $\mathcal{D}_t$. Then, the posterior distribution over $(\boldsymbol\theta, \mathbf{w})$ is given by:
\begin{equation}
\begin{aligned}
\label{eq:full_posterior}
    p(\boldsymbol\theta, \mathbf{w} \mid \mathcal{D}_t, \mathcal{D}_m) &= \frac{p(\boldsymbol\theta, \mathcal{D}_m \mid \mathbf{w}, \mathcal{D}_t) p(\mathbf{w})}{p(\mathcal{D}_m \mid \mathcal{D}_t)} \\
    &\propto p(\boldsymbol\theta \mid \mathbf{w}, \mathcal{D}_t) \, p(\mathcal{D}_m \mid \boldsymbol\theta) \, p(\mathbf{w}),
\end{aligned}
\end{equation}
where $p(\boldsymbol\theta \mid \mathbf{w}, \mathcal{D}_t)$ denotes the conditional distribution of model parameters given the sample weights and training data, $p(\mathcal{D}_m \mid \boldsymbol\theta)$ encourages $\boldsymbol\theta$ toward fairness-aware generalization, and $p(\mathbf{w})$ is a prior over sample weights (e.g., sparsity-inducing or uniform). This formulation enables learning weights that prioritize training examples most compatible with the fairness-oriented meta-dataset $\mathcal{D}_m$.

\subsection{Efficient Posterior Approximation} 
Inferring the joint posterior in~\cref{eq:full_posterior} is generally intractable, particularly when $\boldsymbol\theta$ and $\mathbf{w}$ are high-dimensional. To enable scalable inference, we approximate $p(\boldsymbol\theta, \mathbf{w} \mid \mathcal{D}_t, \mathcal{D}_m)$ using tractable methods. \citet{xu2024a} proposed using SGLD, which augments stochastic gradient descent with Gaussian noise in each update step to simulate Langevin dynamics and sample from the posterior. While effective in many scenarios, SGLD suffers from slow convergence, sensitivity to step size and noise scale, often resulting in unstable training dynamics.

To overcome these limitations, we adopt SVGD, a deterministic, particle-based variational inference method that approximates the posterior by iteratively updating a set of particles via functional gradients in a reproducing kernel Hilbert space (RKHS)~\citep{NIPS2016_svgd, wei2025personalizedbayesianfederatedlearning}. Each particle represents a ``sample'' from the posterior, and is transported toward high-density regions while maintaining diversity. Unlike SGLD, SVGD avoids the randomness and instability of stochastic samplers while better capturing the complex posterior. Formally, given $M$ particles $\{\mathbf{z}^{(i)}\}_{i=1}^M$, the update rule is:
\begin{equation}
\begin{aligned}
\label{eq:svgd_general}
\mathbf{z}^{(i)} \leftarrow \mathbf{z}^{(i)}  + \frac{\epsilon}{N} \sum_{j=1}^N \Big[  k(\mathbf{z}^{(j)} , &\mathbf{z}^{(i)}) \nabla_{\mathbf{z}^{(j)}} \log p(\mathbf{z}^{(j)})  \\
&+ \nabla_{\mathbf{z}^{(j)}} k(\mathbf{z}^{(j)}, \mathbf{z}^{(i)}) \Big],
\end{aligned}
\end{equation}
where \( k(\cdot,\cdot) \) is a positive-definite kernel that defines particle interactions, and \( \nabla_{\mathbf{z}^{(j)}} \log p(\mathbf{z}^{(j)}) \) denotes the gradient of the log-posterior with respect to particle \( \mathbf{z}^{(j)} \). This update encourages convergence to the posterior while mitigating particle collapse.

\section{Methodology}

In this section, we propose the \textit{Fair Bayesian Data Selection} (Fair-BADS) framework (see \cref{fig:scheme}), which jointly infers model parameters and sample weights with fairness considerations. While Bayesian data selection provides a principled and scalable alternative to bi-level optimization or meta-learning, existing approaches often overlook disparities across demographic groups. As a result, models trained under such frameworks may overfit to majority groups due to issues like class imbalance or group-dependent label bias in $\mathcal{D}_t$, leading to unfair performance across subpopulations. 

To tackle this issue, Fair-BADS explicitly models group-specific posteriors and softly aligns them toward a central distribution. This central distribution serves as the group alignment target, defined via a divergence-based objective across group-specific posteriors. Fairness is then introduced at the distributional level by regularizing the divergence between each group-specific posterior and the central distribution. This allows the model to preserve group-specific signals while emphasizing globally fair samples.

Formally, we partition the training set $\mathcal{D}_t$ into demographic groups and define, for each group $s$, a posterior over model parameters $\boldsymbol \theta$ and sample weights $\mathbf{w}$ is: 
\begin{equation*}
    p_s(\boldsymbol \theta, \mathbf{w} ) \coloneqq p(\boldsymbol \theta, \mathbf{w} \mid \mathcal{D}_t^s, \mathcal{D}_m),
\end{equation*}
where $\mathcal{D}_t^s \subseteq \mathcal{D}_t$ is the subset of $\mathcal{D}_t$ from group $s$ with size $N_s = |\mathcal{D}_t^s|$. To allow alignment across groups with differing sizes, we embed each posterior into a common space of dimension $P + \bar{N}$, where $P$ is the model parameter dimension and $\bar{N} = \max_{s} N_s$. For each group, the weight vector $\mathbf{w}$ is zero-padded to this common dimensionality, allowing all particles to reside in a consistent joint space and enabling consistent divergence computation across groups.

We assume that each demographic group induces a distinct posterior reflecting its statistical characteristics and potential biases. To mitigate inter-group disparities, we introduce a fairness-aware alignment mechanism during the inference by softly aligning the group specific posteriors $\{p_s(\boldsymbol \theta, \mathbf{w})\}_{s=1}^S$ toward the central distribution $p^\star(\boldsymbol \theta, \mathbf{w})$, defined as the minimizer of a divergence-based objective: 
\begin{equation}
\label{eq: barycenter}
    p^\star(\boldsymbol\theta,\mathbf{w}) = \argmin_{p} \sum_{s=1}^S \lambda_s D(p , p_s(\boldsymbol \theta,\mathbf{w})),
\end{equation}
where $D(\cdot,\cdot)$ is a user-specified distributional discrepancy, such as Wasserstein distance, MMD, or $f$-divergence. The coefficients $\lambda_s \in (0,1)$ satisfy $\sum_{s} \lambda_s = 1$ and control the contribution of each group ($\lambda_s = 1/S$ in our setting to ensure equal contribution). This central distribution guides the model to balance fairness across demographic groups under the chosen $D$. Specific instantiations and optimization strategies are discussed in subsequent sections.

\subsection{Inference via SVGD}
To approximate the joint posterior over $(\boldsymbol\theta,\mathbf{w})$, we adopt the SVGD algorithm, inspired by \citet{wei2025personalizedbayesianfederatedlearning}. For each group $s$, we maintain a set of $M$ particles $\{z_s^{(m)}=(\boldsymbol\theta_s^{(m)},\mathbf{w}_s^{(m)})\}^M_{m=1}$, where each particle represents a sample from the group-specific posterior $p_s(\boldsymbol\theta,\mathbf{w})$ and $z_s^{(m)}\in \mathbb{R}^{P+\bar{N}}$. 
According to \cref{eq:full_posterior}, the log-posterior can be decomposed into three terms: 
\begin{equation}
\begin{aligned}
&\log p_s(\boldsymbol\theta, \mathbf{w}) =
	-\underbrace{\sum_{i=1}^{\bar{N}} \sigma(w_i) \cdot \mathcal{L}(f_{\boldsymbol\theta}(\mathbf{x}_i), y_i)}_{\text{weighted training loss}}\\
	& -\underbrace{\sum_{(\mathbf{x}_i, y_i) \in \mathcal{D}_m} \mathcal{L}(f_{\boldsymbol\theta}(\mathbf{x}_i), y_i)}_{\text{meta loss}}
	+ \underbrace{\left( \sum_{i=1}^{\bar{N}} \sigma(w_i) - \beta \bar{N} \right)^2}_{\text{weight prior (soft constraint)}},
\end{aligned}
\label{eq: -logp}
\end{equation}
where $\mathcal{L}$ denotes the cross-entropy loss, and $\sigma$ is the sigmoid function used to constrain each weight to $(0,1)$. Following \citet{xu2024a}, we define $p(\mathbf{w})$ implicitly via a soft prior, implemented as a regularization term that encourages the average weight to remain close to a predefined sparsity level $\beta$. Each particle is then updated using the SVGD (\cref{eq:svgd_general}):
\begin{align}
    \mathbf{z}_s^{(m)} \leftarrow \mathbf{z}_s^{(m)}+\frac{\epsilon}{M}\sum^M_{l=1} [ &k(\mathbf{z}_s^{(l)},\mathbf{z}_s^{(m)})\cdot \nabla_{\mathbf{z}_{s}^{(l)}} \log p_s(\mathbf{z}_{s}^{(l)}) \nonumber \\
    & +\nabla_{\mathbf{z}_{s}^{(m)}} k(\mathbf{z}_{s}^{(l)},\mathbf{z}_s^{(m)})],
\label{eq: svgd_update}
\end{align}
where $\mathbf{z}=(\boldsymbol{\theta},\mathbf{w}) \in \mathbb{R}^{P+\bar N}$ denote a sample in the padded parameter-weight space, $k(\cdot, \cdot)$ is a kernel function defined over the joint space to ensure smoothness and diversity across particles and the gradient term is: 
\begin{equation*}
    \begin{aligned}
        &\nabla_{\mathbf{z}}\log p_s(\mathbf{z}) \\
        &= \begin{bmatrix}
            \nabla_{\boldsymbol\theta} \log p(\boldsymbol\theta \mid \mathbf{w}, \mathcal{D}_t^s)+\nabla_{\boldsymbol \theta} \log p(\mathcal{D}_m \mid \boldsymbol\theta)\\
            \nabla_{\mathbf{w}} \log p(\boldsymbol\theta \mid \mathbf{w}, \mathcal{D}_t^s)+\nabla_{\mathbf{w}} \log p(\mathbf{w})
        \end{bmatrix}.
    \end{aligned}
\end{equation*}

After completing the SVGD updates for each group, we obtain particle sets $\{ z_s^{(m)} \}_{m=1}^M$, which are used to construct an empirical approximation of the group-specific posterior: 
\begin{equation*}
       \text{Group-specific Posterior:} \quad \Tilde{p}_s(\mathbf{z}) = \frac{1}{M} \sum_{m=1}^M \delta(\mathbf{z} - \mathbf{z}_s^{(m)}),   
\end{equation*}
where $\delta(\cdot)$ denotes the Dirac delta function, and $\Tilde{p}$ indicates that it is an empirical distribution supported on discrete particles. 
To encourage fairness across groups, we aim to align $\Tilde{p}_s(\mathbf{z})$ toward a central distribution: 
\begin{equation*}
    \text{Central Distribution:} \quad \Tilde{p}^\star(\mathbf{z}) = \frac{1}{M} \sum_{m=1}^M \delta(\mathbf{z} - \bar{\mathbf{z}}^{(m)}), 
\end{equation*}
where the central distribution is represented by a set of particles $\{\bar{\mathbf{z}}^{(m)}\}_{m=1}^M$, which serve as discrete support points summarizing the shared structure across all group-specific posteriors. Each particle $\bar{\mathbf{z}}^{(m)}$ lies in the same space $\mathbb{R}^{P + \bar{N}}$ as the group particles, allowing consistent comparison and alignment across distributions. 
These central particles are later obtained via central distribution computation (see \cref{central_dist}), where we minimize a chosen distributional discrepancy to update $\{\bar{\mathbf{z}}^{(m)}\}_{m=1}^M$. We will detail this computation and update procedure in the following.

To incorporate fairness into the posterior inference, we modify the SVGD update by replacing $\nabla_{\mathbf{z}_{s}^{(l)}} \log p_s(\mathbf{z}_{s}^{(l)})$ in \cref{eq: svgd_update} with $\nabla_{\mathbf{z}_{s}^{(l)}} \log p_\text{fair}(\mathbf{z}_{s}^{(l)})$ that is defined as: 
\[
\log p_\text{fair}(\mathbf{z}) := \log p_s(\mathbf{z}) + \log p^\star(\mathbf{z}),
\]
where a regularization term $\log p^\star(\mathbf{z})$ is introduced to softly encourage each group-specific posterior $\Tilde{p}_s(\mathbf{z})$ to align with a shared central distribution $\Tilde{p}^\star(\mathbf{z})$. 
This design guides particle updates to not only fit the group-specific posteriors but also remain close to the central distribution, thereby promoting fairness at the population level.


\subsection{Computation of Central Distribution}
\label{central_dist}

To compute the central distribution, we solve a divergence minimization problem that aligns group-specific posteriors toward a central distribution, as defined in \cref{eq: barycenter}. We consider three representative divergence measures: Wasserstein distance, MMD, and \( f \)-divergence. 



\paragraph{Wasserstein Distance.}
The Wasserstein distance \( W_2(\cdot,\cdot) \) between two particle-based distributions is relatively easy to compute. 
To measure the discrepancy between the central distribution \( \Tilde{p} \) and a group posterior \( \Tilde{p}_s \), we define a cost matrix \( \mathbf{C}_s \in \mathbb{R}^{M \times M}_+ \), where each element is $\mathbf{C}_s[i,j] = \|\mathbf{z}_s^{(i)} - \bar{\mathbf{z}}^{(j)}\|_2^2$, representing the squared L2 cost between group and central particles.
A transport plan \( \mathbf{T}_s \in \mathbb{R}_+^{M \times M} \) specifies the amount of probability mass transported from \( \mathbf{z}_s^{(i)} \) to \( \bar{\mathbf{z}}^{(j)} \), subject to uniform marginal constraints:
\begin{equation}
\begin{gathered}
    W_2^2(\Tilde{p}, \Tilde{p}_s) = \min_{\mathbf{T}_s} \left\langle \mathbf{C}_s, \mathbf{T}_s \right\rangle_F, \\
    \text{s.t.} \quad \mathbf{T}_s \mathbf{1} = \tfrac{1}{M} \mathbf{1}, \quad \mathbf{T}_s^\top \mathbf{1} = \tfrac{1}{M} \mathbf{1},
\end{gathered}
\label{linear_prog}
\end{equation}
where \( \langle \cdot, \cdot \rangle_F \) denotes the Frobenius inner product and \( \mathbf{1} \in \mathbb{R}^M \) is the all-ones vector.
After solving \cref{linear_prog} and obtaining the optimal plan \( \mathbf{T}_s^\star \), the central distribution is computed by minimizing the weighted sum of Wasserstein distances:
\begin{equation*}
\Tilde{p}^\star = \argmin_{\{\bar{\mathbf{z}}^{(m)}\}_{m=1}^M} \sum_{s=1}^S \lambda_s \left\langle \mathbf{C}_s, \mathbf{T}_s^\star \right\rangle_F.
\end{equation*}
Since the objective is quadratic in \( \{\bar{\mathbf{z}}^{(m)}\} \), the closed-form solution for the central particles exists: 
\begin{equation*}
\bar{\mathbf{Z}}^\star = \sum_{s=1}^S \lambda_s \mathbf{T}_s^\star \mathbf{Z}_s \operatorname{diag}(M^{-1}),
\label{eq: updated_average_z}
\end{equation*}
where \( \mathbf{Z}_s \in \mathbb{R}^{M \times (P + \bar{N})} \) stacks group-specific particles, and \( \bar{\mathbf{Z}}^\star \in \mathbb{R}^{M \times (P + \bar{N})} \) denotes the barycenter particles.

\paragraph{MMD.}
The MMD distance between two particle-based distributions can be computed via the kernel trick, using a kernel function \( \Tilde{k}(\cdot, \cdot) \) to measure the discrepancy between two sample sets. Specifically, the MMD between the central distribution \( \Tilde{p} \) and a group posterior \( \Tilde{p}_s \) is: 
\begin{equation*}
\begin{aligned}
\text{MMD}^2(\Tilde{p}, \Tilde{p}_s) &= \frac{1}{M^2} \sum_{i,j} \Tilde{k}(\bar{\mathbf{z}}^{(i)}, \bar{\mathbf{z}}^{(j)}) + \frac{1}{M^2} \sum_{i,j} \Tilde{k}(\mathbf{z}_s^{(i)}, \mathbf{z}_s^{(j)}) \\
& - \frac{2}{M^2} \sum_{i,j} \Tilde{k}(\bar{\mathbf{z}}^{(i)}, \mathbf{z}_s^{(j)}). 
\end{aligned}
\end{equation*}
The central distribution is obtained by minimizing the weighted sum of the MMDs across all groups: 
\begin{equation*}
\Tilde{p}^\star = \argmin_{\{\bar{\mathbf{z}}^{(m)}\}_{m=1}^M} \sum_{s=1}^S \lambda_s \text{MMD}^2(\Tilde{p}, \Tilde{p}_s).
\end{equation*}
Due to the presence of the kernel, the objective no longer admits a closed-form solution. Therefore, the optimal central particles $\bar{\mathbf{Z}}^\star$ must be obtained via gradient descent.

\paragraph{$f$-divergence.} 

The \( f \)-divergence is a general class of divergence measures, defined as:
\[
D_f(p \,\|\, q) = \int q(\mathbf{z}) \, f\left( \frac{p(\mathbf{z})}{q(\mathbf{z})} \right) \, d\mathbf{z},
\]
where \( f \) is a convex function. 
By choosing different \( f \), we recover various divergence measures. For example, when \( f(t) = t \log t \), we obtain the KL divergence; when \( f(t) = -\log t \), we get the reverse KL divergence; and when \( f(t) = t \log \frac{2t}{t+1} + \log \frac{2}{t+1} \), we obtain the JS divergence. 

The $f$-divergence between two particle-based distributions is generally intractable due to the need to integrate density ratios. We approximate it using kernel density estimation (KDE). Given a kernel function \( \Tilde{k}(\cdot, \cdot) \) and bandwidth \( h > 0 \), the KDEs for \( \Tilde{p} \) and \( \Tilde{p}_s \) are:
\[
\hat{p}(\mathbf{z}) = \frac{1}{M} \sum_{i=1}^{M} \Tilde{k}_h(\mathbf{z}, \bar{\mathbf{z}}^{(i)}), \quad \hat{p}_s(\mathbf{z}) = \frac{1}{M} \sum_{j=1}^{M} \Tilde{k}_h(\mathbf{z}, \mathbf{z}_s^{(j)}). 
\]
Then the \( f \)-divergence between \( \Tilde{p} \) and \( \Tilde{p}_s \) is approximated as: 
\[
D_f(\Tilde{p} \| \Tilde{p}_s) \approx \frac{1}{M} \sum_{j=1}^{M} f\left( \frac{\hat{p}(\mathbf{z}_s^{(j)})}{\hat{p}_s(\mathbf{z}_s^{(j)}) + \epsilon} \right), 
\]
where \( \epsilon > 0 \) is a small constant added for numerical stability. 
The optimal central particles \( \bar{\mathbf{Z}}^\star \) are then obtained by minimizing the weighted sum of $f$-divergences across all groups using gradient descent.

\section{Theoretical Analysis}
\label{subsec:theory}
We present the theoretical guarantee for Fair-BADS from two perspectives: (i) a \emph{discrepancy transfer} bound, showing that evaluating the model on the empirical central distribution $\tilde p^\star$ approximates the average group risk; and (ii) a \emph{group fairness disparity} bound, demonstrating that performance gaps across groups are controlled when their posteriors align with the shared central distribution. 
Define 
\[
R_s(p) \triangleq \mathbb{E}_{\mathbf{z}\sim p}\big[\mathcal{L}_s(\mathbf{z})\big], \qquad
R(p) \triangleq \mathbb{E}_{\mathbf{z}\sim p}\big[\mathcal{L}(\mathbf{z})\big],
\]
where $\mathcal{L}$ is the loss function used in \cref{eq: -logp}, and $R(\cdot)$ denotes the expected risk. The subscript $s$ indicates it is computed w.r.t. group $s$ only. 
And we use the following discrepancy-specific regularity assumptions. 

\paragraph{(A1) Loss Regularity.} For each group $s$, the loss $\mathcal{L}_s$ is bounded and satisfies:
\[
\big|\mathbb{E}_{p}\mathcal{L}_s(\mathbf{z})-\mathbb{E}_{q}\mathcal{L}_s(\mathbf{z})\big|\le C_s  D(p,q),
\]
with $C_s$ depending on the choice of $D$:
\begin{itemize}
	\item $C_s=L_s$ if $D=W_2$ and $\mathcal{L}_s$ is $L_s$–Lipschitz;
	\item $C_s=\|\mathcal{L}_s\|_{\mathcal{H}_{\Tilde{k}}}$ if $D=\mathrm{MMD}_{\Tilde{k}}$ and $\mathcal{L}_s\in\mathcal{H}_{\Tilde{k}}$;
	\item $C_s=B_s\sqrt{2c_f}$ if $D=D_f$ and $\mathcal{L}_s\in[0,B_s]$,
where $\mathcal{H}_{\Tilde{k}}$ is a reproducing kernel Hilbert space, $B_s$ and $c_f$ are some constants whose definitions are provided in \cref{app:theory_full}. 
\end{itemize}


\paragraph{(A2) Cross–Group Compatibility.} There exists $K<\infty$ such that for any $\mathbf{z}$ and $s,s'$,
$|\mathcal{L}_s(\mathbf{z})-\mathcal{L}_{s'}(\mathbf{z})|\le K$.

\begin{theorem}[Discrepancy Transfer Bound]
\label{thm:transfer}
Let $\tilde{p}^\star$ be the empirical central distribution minimizing \cref{eq: barycenter} and define $\bar{R}\triangleq\sum_{s=1}^S \lambda_s R_s(\tilde{p}_s)$. Under (A1),
\begin{equation}\label{eq:bound1}
	\big|R(\tilde{p}^\star)-\bar{R}\big|
	 \le 
	\sum_{s=1}^S\lambda_s  C_s  D(\tilde{p}_s,\tilde{p}^\star).
\end{equation}
Concretely:
\begin{align*}
	&\text{(Wasserstein)}\quad 
	\big|R(\tilde{p}^\star)-\bar{R}\big|
	\le \sum_s \lambda_s L_s  W_2(\tilde{p}_s,\tilde{p}^\star). \\
	& \text{(MMD)}\quad
	\big|R(\tilde{p}^\star)-\bar{R}\big|
	\le \sum_s \lambda_s \|\mathcal{L}_s\|_{\mathcal{H}_{\Tilde{k}}}  \mathrm{MMD}_{\Tilde{k}}(\tilde{p}_s,\tilde{p}^\star). \\
	& \text{($f$-divergence) }
	\big|R(\tilde{p}^\star)-\bar{R}\big|
	\le \sum_s \lambda_s B_s \sqrt{2c_f} D_f(\tilde{p}_s\|\tilde{p}^\star). 
\end{align*}
\end{theorem}

\begin{theorem}[Group Fairness Disparity Bound]
	\label{thm:disparity-vanish}
    Suppose Fair-BADS is run for $t=0,1,\ldots$, producing group posteriors $\{\tilde p_s^{(t)}\}_{s=1}^S$ and central $\tilde p_\star^{(t)}$. Define the \emph{effective cross–group gap} restricted to the support of the current central:
	\[
	K_{\mathrm{eff}}(t)  \triangleq  \sup_{\mathbf{z} \in \mathrm{supp}(\tilde p_\star^{(t)})} \max_{s,s'} \big|\mathcal{L}_s(\mathbf{z}) - \mathcal{L}_{s'}(\mathbf{z})\big|.
	\]
	Then for any $s,s'$,
	\begin{equation}		\label{eq:disparity-time}
		\begin{aligned}
					&\Big|\E_{\tilde p_s^{(t)}} \mathcal{L}_s - \E_{\tilde p_{s'}^{(t)}} \mathcal{L}_{s'}\Big|
			\\ 
			\le&
			C_s D(\tilde p_s^{(t)},\tilde p_\star^{(t)}) + C_{s'} D(\tilde p_{s'}^{(t)},\tilde p_\star^{(t)}) + K_{\mathrm{eff}}(t)\\
			\le & 	2 C_{\max}  \max_{s} D(\tilde{p}_s,\tilde{p}^\star) + K_{\mathrm{eff}}(t),
		\end{aligned}
	\end{equation}
	where $C_{\max}=\max_s C_s$.
\end{theorem} 
The term \( K_{\mathrm{eff}}(t) \) reflects intrinsic group-level difficulty and cannot be fully eliminated by alignment. However, it typically \emph{decreases over iterations} as (i) \( \tilde p_t^\star \) concentrates on low-loss parameter regions, and (ii) learned weights \( \mathbf{w} \) down-weight disparity-inducing samples. \cref{app:theory_full} provides sufficient conditions for \( K_{\mathrm{eff}}(t) \to 0 \) along with full assumptions and proofs.

\begin{table*}[ht]
\centering
\resizebox{0.9\textwidth}{!}{
\begin{tabular}{l|cccc|cccc}
\toprule
\multirow{2}{*}{\textbf{Method}} 
& \multicolumn{4}{c|}{\textbf{Bias amount: 0.2}} 
& \multicolumn{4}{c}{\textbf{Bias amount: 0.4}} \\
\cmidrule(lr){2-5} \cmidrule(lr){6-9}
& \textbf{ACC}($\uparrow$) & \textbf{DP}($\downarrow$) & \textbf{DDP}($\downarrow$) & \textbf{EO}($\downarrow$) 
& \textbf{ACC}($\uparrow$) & \textbf{DP}($\downarrow$) & \textbf{DDP}($\downarrow$) & \textbf{EO}($\downarrow$) \\
\midrule
 & \multicolumn{8}{c}{\textbf{UTKFace}} \\
\midrule
ERM              
    & 0.848$_{\pm 0.005}$ &  \textbf{0.062$_{\pm 0.011}$} & 0.027$_{\pm 0.011}$ & \textbf{0.088$_{\pm 0.027}$} & 0.770$_{\pm 0.009}$ &  0.252$_{\pm 0.012}$ & 0.051$_{\pm 0.027}$ & 0.291$_{\pm 0.032}$ \\
FairBatch  
    & 0.803$_{\pm 0.034}$ & 0.082$_{\pm 0.024}$ & 0.048$_{\pm 0.033}$ & 0.123$_{\pm 0.069}$ & 0.721$_{\pm 0.016}$ & 0.254$_{\pm 0.025}$ & 0.058$_{\pm 0.030}$ & 0.294$_{\pm 0.041}$ \\
FERM      
    & 0.843$_{\pm 0.006}$ & 0.069$_{\pm 0.002}$ & 0.030$_{\pm 0.014}$ & 0.095$_{\pm 0.029}$ & 0.780$_{\pm 0.003}$ & 0.225$_{\pm 0.020}$ & 0.059$_{\pm 0.018}$ & 0.296$_{\pm 0.057}$  \\
BLO              
    & 0.805$_{\pm 0.005}$ & 0.099$_{\pm 0.007}$ & 0.028$_{\pm 0.014}$ & 0.154$_{\pm 0.019}$ & 0.726$_{\pm 0.008}$ & 0.249$_{\pm 0.006}$ & 0.057$_{\pm 0.023}$ & \textbf{0.286$_{\pm 0.027}$} \\
BADS             
    & 0.816$_{\pm 0.013}$ & 0.068$_{\pm 0.013}$ & 0.050$_{\pm 0.013}$ & 0.141$_{\pm 0.023}$ & 0.751$_{\pm 0.015}$ & 0.225$_{\pm 0.028}$ & 0.053$_{\pm 0.014}$ & 0.303$_{\pm 0.066}$  \\
Fair-BADS-W  
    & \textbf{0.851$_{\pm 0.002}$} & \textbf{0.062$_{\pm 0.006}$} & 0.026$_{\pm 0.003}$ & 0.116$_{\pm 0.011}$ & \textbf{0.787$_{\pm 0.007}$} & 0.219$_{\pm 0.021}$ & 0.060$_{\pm 0.028}$ & 0.309$_{\pm 0.040}$ \\
Fair-BADS-M  
    & 0.849$_{\pm 0.009}$ & 0.073$_{\pm 0.012}$ & 0.031$_{\pm 0.006}$ & 0.132$_{\pm 0.009}$ & 0.781$_{\pm 0.007}$ & 0.214$_{\pm 0.026}$ & 0.050$_{\pm 0.013}$ & 0.295$_{\pm 0.035}$ \\
Fair-BADS-F  
    & 0.849$_{\pm 0.006}$ & 0.069$_{\pm 0.014}$ & \textbf{0.023$_{\pm 0.002}$} & 0.286$_{\pm 0.014}$ & 0.786$_{\pm 0.007}$ & \textbf{0.211$_{\pm 0.021}$} & \textbf{0.043$_{\pm 0.012}$} & \textbf{0.286$_{\pm 0.029}$} \\
\midrule
& \multicolumn{8}{c}{\textbf{LFW-A}} \\
\midrule
ERM  
    & 0.884$_{\pm 0.005}$ & 0.142$_{\pm 0.020}$ & 0.011$_{\pm 0.006}$ & 0.041$_{\pm 0.020}$ & 0.821$_{\pm 0.019}$ & 0.273$_{\pm 0.025}$ & 0.079$_{\pm 0.040}$ & 0.192$_{\pm 0.048}$ \\
FairBatch        
    & 0.889$_{\pm 0.008}$ & 0.131$_{\pm 0.006}$ & 0.016$_{\pm 0.015}$ & 0.051$_{\pm 0.015}$ & 0.780$_{\pm 0.014}$ & 0.251$_{\pm 0.009}$ & 0.080$_{\pm 0.016}$ & 0.182$_{\pm 0.018}$ \\
FERM      
    & 0.860$_{\pm 0.056}$ & 0.142$_{\pm 0.070}$ & 0.013$_{\pm 0.003}$ & 0.035$_{\pm 0.013}$ & 0.829$_{\pm 0.013}$ & 0.237$_{\pm 0.042}$ & 0.028$_{\pm 0.021}$ & 0.130$_{\pm 0.044}$ \\
BLO              
    & 0.888$_{\pm 0.002}$ & 0.147$_{\pm 0.010}$ & 0.023$_{\pm 0.012}$ & 0.068$_{\pm 0.016}$ & 0.798$_{\pm 0.024}$ & 0.253$_{\pm 0.024}$ & 0.084$_{\pm 0.015}$ & 0.185$_{\pm 0.028}$ \\
BADS         
    & 0.884$_{\pm 0.004}$ & 0.140$_{\pm 0.025}$ & 0.014$_{\pm 0.012}$ & 0.056$_{\pm 0.023}$ & 0.834$_{\pm 0.018}$ & 0.217$_{\pm 0.019}$ & 0.031$_{\pm 0.017}$ & 0.122$_{\pm 0.025}$ \\
Fair-BADS-W  
    & \textbf{0.902$_{\pm 0.011}$} & \textbf{0.129$_{\pm 0.004}$} & \textbf{0.006$_{\pm 0.004}$} & 0.042$_{\pm 0.005}$ & \textbf{0.859$_{\pm 0.006}$} & \textbf{0.162$_{\pm 0.018}$} & \textbf{0.012$_{\pm 0.007}$} & \textbf{0.052$_{\pm 0.019}$} \\
Fair-BADS-M  
    & 0.901$_{\pm 0.006}$ & 0.133$_{\pm 0.015}$ & 0.010$_{\pm 0.004}$ & 0.034$_{\pm 0.007}$ & 0.850$_{\pm 0.006}$ & 0.186$_{\pm 0.032}$ & 0.024$_{\pm 0.012}$ & 0.090$_{\pm 0.041}$ \\
Fair-BADS-F  
    & 0.900$_{\pm 0.003}$ & 0.132$_{\pm 0.014}$ & 0.014$_{\pm 0.003}$ & \textbf{0.033$_{\pm 0.009}$} & \textbf{0.859$_{\pm 0.018}$} & 0.189$_{\pm 0.013}$ & 0.014$_{\pm 0.007}$ & 0.079$_{\pm 0.021}$ \\
\midrule
& \multicolumn{8}{c}{\textbf{FairFace}} \\
\midrule
ERM              
    & 0.716$_{\pm 0.014}$ & 0.170$_{\pm 0.038}$ & 0.045$_{\pm 0.016}$ & 0.198$_{\pm 0.022}$ & 0.656$_{\pm 0.007}$ & 0.392$_{\pm 0.020}$ & 0.030$_{\pm 0.003}$ & 0.402$_{\pm 0.022}$  \\
FairBatch        
     & 0.685$_{\pm 0.011}$ & 0.139$_{\pm 0.016}$ & 0.044$_{\pm 0.015}$ & 0.168$_{\pm 0.003}$ & 0.629$_{\pm 0.003}$ & 0.328$_{\pm 0.019}$ & 0.044$_{\pm 0.018}$ & 0.357$_{\pm 0.036}$ \\
FERM         
    & 0.699$_{\pm 0.006}$ & 0.416$_{\pm 0.002}$ & \textbf{0.026$_{\pm 0.013}$} & 0.156$_{\pm 0.032}$ & 0.628$_{\pm 0.006}$ & 0.416$_{\pm 0.002}$ & 0.026$_{\pm 0.016}$ & 0.422$_{\pm 0.017}$ \\
BLO         
    & 0.680$_{\pm 0.005}$ & 0.128$_{\pm 0.002}$ & 0.048$_{\pm 0.020}$ & 0.162$_{\pm 0.021}$ & 0.618$_{\pm 0.002}$ & \textbf{0.320$_{\pm 0.014}$} & 0.054$_{\pm 0.008}$ & \textbf{0.335$_{\pm 0.007}$} \\
BADS        
    & 0.661$_{\pm 0.011}$ & 0.159$_{\pm 0.023}$ & 0.055$_{\pm 0.010}$ & 0.203$_{\pm 0.031}$ & 0.632$_{\pm 0.012}$ & 0.369$_{\pm 0.096}$ & 0.047$_{\pm 0.008}$ & 0.400$_{\pm 0.097}$ \\
Fair-BADS-W  
    & 0.718$_{\pm 0.009}$ & 0.140$_{\pm 0.033}$ & 0.038$_{\pm 0.011}$ & 0.165$_{\pm 0.021}$ & 0.662$_{\pm 0.009}$ & 0.341$_{\pm 0.038}$ & 0.026$_{\pm 0.001}$ & 0.350$_{\pm 0.038}$ \\
Fair-BADS-M  
    & \textbf{0.719$_{\pm 0.008}$} & 0.141$_{\pm 0.031}$ & 0.040$_{\pm 0.011}$ & 0.168$_{\pm 0.022}$ & 0.660$_{\pm 0.009}$ & 0.342$_{\pm 0.036}$ & 0.027$_{\pm 0.008}$ &  0.353$_{\pm 0.028}$ \\
Fair-BADS-F  
    & \textbf{0.719$_{\pm 0.008}$} & \textbf{0.126$_{\pm 0.035}$} & 0.045$_{\pm 0.012}$ & \textbf{0.156$_{\pm 0.008}$} & \textbf{0.663$_{\pm 0.008}$} & 0.327$_{\pm 0.042}$ & \textbf{0.025$_{\pm 0.005}$} & \textbf{0.335$_{\pm 0.045}$} \\
\bottomrule
\end{tabular}}
\caption{Evaluation results under different bias amount. For Fair-BADS, we report the results using three different variants.}
\label{tab:compare_results_main_table}
\end{table*}

\section{Experiments}
 In the following sections, we first outline the experimental setup, then compare our method with related work across diverse image classification tasks under varying levels of label bias, followed by ablation studies on our selection strategy.

\subsection{Experimental Setup}
For each dataset, we simulate label bias using group-dependent corruption strategies~\citep{unlocking_fairness}. Unless otherwise specified, we use 20 particles per group and set the weight prior strength to \( \beta = 0.005 \). We use the JS divergence as our choice of \( f \)-divergence. The kernel \( k \) in SVGD and \( \Tilde{k} \) in MMD and \( f \)-divergence are both Gaussian kernels with adaptive bandwidth \( h = 0.1 \), numerical stability constant \( \epsilon = 1\text{e-}3 \). The heuristic kernel may degrade in high-dimensional spaces, where norm-regularized~\citep{after-rb-1} or PDE-based kernels~\citep{after-rb-2} provide more robust alternatives.

\paragraph{Datasets.} We evaluate our method on three image datasets: \textbf{UTKFace}~\citep{utkface}, Labeled Faces in the Wild with Attributes (\textbf{LFW-A})~\citep{lfwa_usage,lfw-a}, and \textbf{FairFace}~\citep{karkkainenfairface}. In UTKFace, race is used as the sensitive attribute and gender as the prediction target. For LFW-A, we predict gender and treat ``HeavyMakeup" as the sensitive attribute due to its observed correlation with gender bias. In FairFace, we perform binary gender classification using race as the sensitive variable, grouping individuals as ``White" or ``Black" to evaluate fairness.

\begin{table*}[ht]
    \centering
    \resizebox{1.0\textwidth}{!}{
    \begin{tabular}{c | c c c c| c c c c}
    \toprule
     \multirow{2}{*}{\textbf{Method}}  & \multicolumn{4}{c|}{\textbf{LFW-A (bias amount: 0.2)}} & \multicolumn{4}{c}{\textbf{LFW-A (bias amount: 0.4)}}\\
     \cmidrule(lr){2-5} \cmidrule(lr){6-9}
     &  \textbf{ACC}($\uparrow$) & \textbf{DP}($\downarrow$) & \textbf{DDP}($\downarrow$) & \textbf{EO}($\downarrow$) & \textbf{ACC}($\uparrow$) & \textbf{DP}($\downarrow$) & \textbf{DDP}($\downarrow$) & \textbf{EO}($\downarrow$) \\
    \midrule
    Fair-BADS-W & \textbf{0.891}$_{\pm 0.018}$ & 0.144$_{\pm 0.023}$ & 0.014$_{\pm 0.008}$ & 0.049$_{\pm 0.029}$ & 0.847$_{\pm 0.016}$ & 0.192$_{\pm 0.044}$ & 0.034$_{\pm 0.008}$ & 0.095$_{\pm 0.053}$  \\
    Fair-BADS-M & 0.890$_{\pm 0.017}$ & 0.142$_{\pm 0.026}$ & \textbf{0.010}$_{\pm 0.005}$ & 0.046$_{\pm 0.028}$ & 0.846$_{\pm 0.005}$ & 
    \textbf{0.187}$_{\pm 0.054}$ & \textbf{0.029}$_{\pm 0.015}$ & \textbf{0.090}$_{\pm 0.057}$ \\
    Fair-BADS-F & 0.889$_{\pm 0.020}$ & \textbf{0.133}$_{\pm 0.025}$ & 0.017$_{\pm 0.004}$ & \textbf{0.039}$_{\pm 0.031}$
    & \textbf{0.850}$_{\pm 0.020}$ & 0.198$_{\pm 0.040}$ & \textbf{0.029}$_{\pm 0.015}$ & 0.093$_{\pm 0.049}$\\
    \bottomrule
    \end{tabular}
    }
    \caption{Evaluation results on LFW-A, with CLIP-RN50 used as a zero-shot predictor for meta loss approximation.}
    \label{tab:variant_zero_shot_predictor}
\end{table*}

\begin{table}[ht]
\centering
\resizebox{\columnwidth}{!}{
\begin{tabular}{l|cccc}
\toprule
\textbf{Method} & \textbf{ACC} ($\uparrow$) & \textbf{DP} ($\downarrow$) & \textbf{DDP} ($\downarrow$) & \textbf{EO} ($\downarrow$) \\
\midrule
& \multicolumn{4}{c}{Backbone: ResNet-18}\\
\midrule
Fair-BADS-W 
    & 0.820$_{\pm 0.015}$ & \textbf{0.049}$_{\pm 0.017}$ & 0.031$_{\pm 0.016}$ & 0.002$_{\pm 0.001}$ \\
Fair-BADS-M 
    & 0.819$_{\pm 0.007}$ & 0.051$_{\pm 0.007}$ & \textbf{0.030$_{\pm 0.007}$} & \textbf{0.001}$_{\pm 0.001}$ \\
Fair-BADS-F 
    & \textbf{0.821$_{\pm 0.016}$} & 0.051$_{\pm 0.018}$ & \textbf{0.030$_{\pm 0.017}$} & 0.002$_{\pm 0.001}$ \\
\midrule    
& \multicolumn{4}{c}{Backbone: DenseNet-121}\\
\midrule
Fair-BADS-W 
    & 0.841$_{\pm 0.009}$ & 0.080$_{\pm 0.002}$ & \textbf{0.004}$_{\pm 0.002}$ & \textbf{0.001}$_{\pm 0.001}$ \\
Fair-BADS-M 
    & \textbf{0.845$_{\pm 0.011}$} & 0.081$_{\pm 0.003}$ & 0.008$_{\pm 0.004}$ & 0.003$_{\pm 0.002}$ \\
Fair-BADS-F 
    & 0.837$_{\pm 0.008}$ & \textbf{0.076$_{\pm 0.006}$} & 0.007$_{\pm 0.011}$ & 0.002$_{\pm 0.001}$ \\
\midrule
& \multicolumn{4}{c}{Backbone: ViT-B/16}\\
\midrule
Fair-BADS-W 
    & 0.867$_{\pm 0.014}$ & 0.160$_{\pm 0.005}$ & 0.021$_{\pm 0.006}$ & 0.079$_{\pm 0.005}$ \\
Fair-BADS-M 
    & \textbf{0.874}$_{\pm 0.011}$ & \textbf{0.156$_{\pm 0.007}$} & \textbf{0.018$_{\pm 0.006}$} & \textbf{0.073$_{\pm 0.007}$} \\
Fair-BADS-F 
    & 0.866$_{\pm 0.016}$ & 0.160$_{\pm 0.005}$ & 0.021$_{\pm 0.008}$ & 0.079$_{\pm 0.007}$\\

\bottomrule
\end{tabular}}
\caption{Comparison of Fair-BADS variants across backbones under bias level 0.4.}
\label{tab:distance_backbone_full}
\end{table}

\paragraph{Baselines and Metrics.}
We evaluate our proposed method against several representative baselines, including standard empirical risk minimization (\textbf{ERM}), a sampling-based approach (\textbf{FairBatch}~\citep{Roh0WS21_fairbatch}), an in-processing fairness method using $f$-divergence (\textbf{FERM}~\citep{baharlouei2024fferm}), a reweighting-based data selection method (\textbf{BLO}) and a Bayesian data selection method (\textbf{BADS}~\citep{xu2024a}). For our Fair-BADS, we implement three variants based on different discrepancy measures: Wasserstein distance (Fair-BADS-W), MMD (Fair-BADS-M) and $f$-divergence (Fair-BADS-F). All methods are evaluated under a consistent experimental setup, using both accuracy and fairness metrics (Demographic Parity (\textbf{DP}), Difference in Demographic Parity (\textbf{DDP}) and Equal Opportunity (\textbf{EO})). Each experiment is conducted with three runs, and we report the mean ${\pm}$ standard deviation.

\subsection{Comparison Results}
\cref{tab:compare_results_main_table} summarizes the results across UTKFace, LFW-A and FairFace under varying levels of label bias. On both UTKFace and LFW-A, our Fair-BADS variants consistently achieve the best and competitive accuracy while reducing fairness disparities compared to other baselines. In particular, Fair-BADS-W yields the best overall trade-off, outperforming the original BADS in both accuracy and fairness metrics. Fair-BADS-W shows more stable improvements, though MMD and $f$-divergence variants also perform well. ERM and BLO tend to suffer from increasing fairness gaps under higher bias, while FairBatch and FERM reduce disparities at the cost of performance. Though BLO and BADS are originally designed to handle low quality data via data selection, they do not explicitly address fairness and thus inadvertently reinforce bias by prioritizing samples from the majority group. To validate these improvements, we conduct paired t-tests and find that on UTKFace, Fair-BADS-W significantly outperforms the next-best method (BADS) in both accuracy and fairness across all bias levels ($p < 0.001$). On LFW-A, it also shows significant gains, especially under high bias ($p < 0.01$). On FairFace, Fair-BADS-F and Fair-BADS-M consistently outperform the next-best method, with significant improvements in accuracy and DP ($p < 0.05$).

\subsection{Learning without Meta Dataset}

In scenarios where no explicit meta dataset $\mathcal{D}_m$ is available, we approximate the meta objective using a zero-shot predictor $f^\ast(\mathbf{x})$ trained on external data (e.g., CLIP-RN50). Instead of directly evaluating $p(\mathcal{D}_m \mid \boldsymbol\theta)$, we estimate it as:
\begin{equation}
\log p(\mathcal{D}_m \mid \boldsymbol\theta) \approx - \mathrm{KL}\big[ p_{\boldsymbol\theta}(\mathbf{y} \mid \mathbf{x}) || p(\mathbf{y} \mid f^\ast(\mathbf{x})) \big],
\end{equation}
where $p_{\boldsymbol\theta}(\mathbf{y} \mid \mathbf{x})$ is the model’s output distribution and $p(\mathbf{y} \mid f^\ast(\mathbf{x}))$ is the pseudo label distribution induced by the zero-shot predictor. This KL divergence acts as a surrogate meta loss, allows us to avoid explicit collection of a meta set. To compute it, we reserve a small fraction (1\%) of the training data as $\mathcal{D}_m^{\text{pseudo}}$, which is excluded from training loss throughout training. As shown in \cref{tab:variant_zero_shot_predictor}, even under the situation when meta set is not available, our method still outperforms all baseline approaches in both accuracy and fairness metrics, despite showing slightly lower performance compared to the performance use explicit meta set. This highlights the framework’s practical advantage in settings where collecting a clean meta set is infeasible.

\subsection{Ablation Studies}
To assess architectural impact, we compare three backbones: ResNet-18~\citep{DBLP:journals/corr/HeZRS15}, DenseNet-121~\citep{huang2017densely}, and ViT-B/16~\citep{dosovitskiy2021an}. 
As shown in \cref{tab:distance_backbone_full}, ViT-B/16 yields the best accuracy, while all variants maintain low fairness metrics. This confirms that our method generalizes well across architectures and that MMD and Wasserstein distances provide more stable fairness control than $f$-divergence.


Beyond architectural variations, we also examine how fairness emerges throughout training. In \cref{fig:weight_epochs}, the KDE plot (left) shows near-identical sample weight distributions across groups by the final epoch, indicating unbiased data selection. The Wasserstein distance (right) decreases during training, confirming that group posteriors align progressively. This supports the effectiveness of barycenter-based alignment in improving fairness.

\begin{figure}
    \centering
    \includegraphics[width=0.95\linewidth]{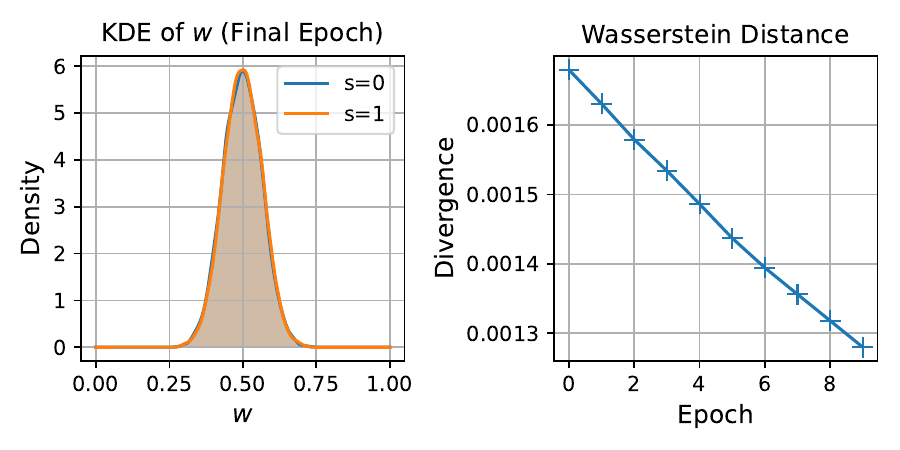}
    \caption{Comparison of sample weight distributions across demographic groups. Left: KDE of sample weights $\mathbf{w}$ at the final training epoch for groups $s=0$ and $s=1$. Right: Wasserstein distance between group-specific weight distributions over training epochs.}
    \label{fig:weight_epochs}
\end{figure}

\section{Conclusions}
We propose \textit{Fair-BADS}, a framework that addresses fairness by combining group-specific inference with distributional alignment. Unlike prior methods that overlook group disparities, we model group-specific posteriors and align them via a shared central distribution, acting as a soft regularizer in SVGD. This approach ensures inter-group consistency without adversarial training or hard constraints. Our particle-based inference is scalable and naturally promotes distributional fairness. Experiments demonstrate improved fairness with strong task performance. Future directions include extending to continuous attributes, dynamic group discovery, and complex tasks like language generation.

\section{Acknowledgments}
This work was supported by the NSFC Projects (No. 62506069, No. 62576346), the MOE Project of Key Research Institute of Humanities and Social Sciences (22JJD110001), and Beijing Advanced Innovation Center for Future Blockchain and Privacy Computing. We also thank Shanghai Institute for Mathematics and Interdisciplinary Sciences (SIMIS) for their financial support under grant number SIMIS-ID-2025-EM.

\bibliography{aaai2026}

\clearpage

\appendix

\section{Theory: Full Assumptions, Lemmas, and Proofs}
\label{app:theory_full}

This appendix provides all technical details that underlie Sec.~\ref{subsec:theory}. Throughout we write $z=(\boldsymbol\theta,\mathbf w)$, $p_s$ for the group posterior, $p^\star$ for the barycenter, and $D$ for the chosen discrepancy. We separate the analysis for $D\in\{W_2,\mathrm{MMD}_k,D_f\}$ and then present unified statements.

\subsection{Assumptions}
\label{app:assumpandpadding}

\begin{assumption}[Loss regularity w.r.t.  $D$]
	\label{assump:loss-regularity}
	For each $s\in S$, the group loss $\mathcal{L}_s:\Xi\to[0,\infty)$ is measurable and satisfies, for any distributions $p,q$ on $\Xi$,
	\[
	\Big|\mathbb{E}_{p}\mathcal{L}_s(z)-\mathbb{E}_{q}\mathcal{L}_s(z)\Big|\le C_s  D(p,q),
	\]
	with the following instantiations:
	\vspace{1mm}
	\begin{itemize}
		\item \textbf{Wasserstein $W_2$:} $\mathcal{L}_s$ is $L_s$–Lipschitz under the metric $d(\cdot,\cdot)$ inducing $W_2$; then $C_s=L_s$ by Kantorovich–Rubinstein type arguments one can prove $W_1$ then adapting to $W_2$.
		\item \textbf{MMD:} $\mathcal{L}_s\in\mathcal{H}_k$ with $\|\mathcal{L}_s\|_{\mathcal{H}_k}\le C_s$; then $C_s=\|\mathcal{L}_s\|_{\mathcal{H}_k}$ and the MMD duality yields the bound.
		\item \textbf{$f$‑divergence:} $\mathcal{L}_s\in[0,B_s]$; by Pinsker‑type inequalities, $|\mathbb{E}_p f - \mathbb{E}_q f|\le B_s \sqrt{2c_f D_f(p\|q)}$ where $c_f$ depends on the chosen $f$ (e.g., $c_f{=}1$ for KL  divergence, $c_f = 2$ for JS divergence  and $c_f = 1/(1+t)$ for $\chi^2$ divergence). 
	\end{itemize}
\end{assumption}

\begin{remark}
	The constants $C_s$ make explicit how the divergence choice affects tightness: $W_2$ gives linear (but potentially expensive) bounds, MMD/IPMs offer linear bounds and are kernel‑amenable, while $f$‑divergences yield $\sqrt{\cdot}$–type bounds.
\end{remark}
\begin{assumption}[Cross–group compatibility]
	\label{assump:cross-group}
	There exists $K<\infty$ such that $|\mathcal{L}_s(z)-\mathcal{L}_{s'}(z)|\le K$ for all $z$ and $s,s'\in S$.
\end{assumption}

\paragraph{Padding for unequal $N_s$.}
Let $\mathcal{P}_s:\mathbb{R}^{P+N_s}\to\mathbb{R}^{P+\bar N}$ (with $\bar N = \max_s N_s$) be the zero‑padding operator used in the main text. We use the following fact.

\begin{proposition}[Divergence preservation under padding]
	\label{prop:padding}
	Let $p_s,q_s$ be distributions on $\mathbb{R}^{P+N_s}$, and $\bar p_s=(\mathcal{P}_s)_\# p_s$, $\bar q_s=(\mathcal{P}_s)_\# q_s$ their pushforwards. Then:
	\begin{itemize}
		\item $W_2(\bar p_s,\bar q_s)=W_2(p_s,q_s)$;
		\item $\mathrm{MMD}_k(\bar p_s,\bar q_s)=\mathrm{MMD}_k(p_s,q_s)$ for translation–invariant kernels;
		\item $D_f(\bar p_s\|\bar q_s)=D_f(p_s\|q_s)$ for any $f$–divergence.
	\end{itemize}
\end{proposition}

\begin{proof}[Proof Sketch for Proposition~\ref{prop:padding}]
	For Wasserstein distance, the optimal transport plan between padded distributions can be constructed from the optimal plan between original distributions, preserving the transport cost due to the isometry property.
	
	For MMD with translation-invariant kernels, the kernel evaluations depend only on distances, which are preserved under padding.
	
	For f-divergences, recall that $D_f(p \| q) = \int f\left(\frac{dp}{dq}\right) dq$ when $p \ll q$. The padded distributions have densities that factorize as:
	\begin{equation}
		\bar{p}_s(z) = p_s(\mathcal{Q}_s(z)) \cdot \mathbf{1}_{\text{supp}(\mathcal{P}_s)}(z)
	\end{equation}
	where the indicator function ensures the measure is supported on the padded subspace. The density ratio is preserved:
	\begin{equation}
		\frac{d\bar{p}_s}{d\bar{q}_s}(z) = \frac{p_s(\mathcal{Q}_s(z))}{q_s(\mathcal{Q}_s(z))} = \frac{dp_s}{dq_s}(\mathcal{Q}_s(z))
	\end{equation}
	on the support of $\mathcal{P}_s$. Therefore:
	\begin{align}
		D_f(\bar{p}_s \| \bar{q}_s) &= \int_{\text{supp}(\mathcal{P}_s)} f\left(\frac{d\bar{p}_s}{d\bar{q}_s}\right) d\bar{q}_s\\
		&= \int_{\mathbb{R}^{P + N_s}} f\left(\frac{dp_s}{dq_s}\right) dq_s = D_f(p_s \| q_s)
	\end{align}
	
	Examples of preserved f-divergences include:
	\begin{itemize}
		\item KL divergence: $f(t) = t\log t$
		\item JS divergence: $f(t) = t\log t - (t+1)\log\frac{t+1}{2}$
		\item $\chi^2$ divergence: $f(t) = (t-1)^2$
		\item $\alpha$-divergence: $f(t) = \frac{t^\alpha - \alpha t + \alpha - 1}{\alpha(\alpha-1)}$
	\end{itemize}
\end{proof}

Hence all our bounds proved in the common (padded) space numerically equal their counterparts in each group’s native space.

\paragraph{Assumptions.}
We use the following discrepancy-specific regularity.
\begin{enumerate}
\item[(A1)] (\textbf{Loss regularity}) For each $s$, the loss $\mathcal{L}_s$ is bounded and satisfies a $D$–Lipschitz–type condition: for all distributions $p,q$,
\[
\big|\mathbb{E}_{p}\mathcal{L}_s(z)-\mathbb{E}_{q}\mathcal{L}_s(z)\big|\le C_s  D(p,q),
\]
where $C_s$ depends on the choice of $D$:
\begin{itemize}
	\item $C_s=L_s$ if $D=W_2$ and $\mathcal{L}_s$ is $L_s$–Lipschitz;
	\item $C_s=\|\mathcal{L}_s\|_{\mathcal{H}_k}$ if $D=\mathrm{MMD}_k$ and $\mathcal{L}_s\in\mathcal{H}_k$;
	\item $C_s=B_s\sqrt{2c_f}$ if $D=D_f$ and $\mathcal{L}_s\in[0,B_s]$ (Pinsker‑type).
\end{itemize}
\item[(A2)] (\textbf{Cross–group compatibility}) There exists $K<\infty$ such that for any $z$ and $s,s'\in S$,
$|\mathcal{L}_s(z)-\mathcal{L}_{s'}(z)|\le K$.
\end{enumerate}

\subsection{Discrepancy Transfer Bound (Theorem~\ref{thm:transfer})}
\label{app:proof-transfer}

\begin{proof}
	Let $\bar R=\sum_{s}\lambda_s R_s(\tilde p_s)$. Then
	\[
	R(\tilde p^\star)-\bar R 
	= \sum_{s}\lambda_s \Big(\mathbb{E}_{\tilde p^\star}\mathcal{L}_s(z)-\mathbb{E}_{\tilde p_s}\mathcal{L}_s(z)\Big).
	\]
	By (A1) for each $s$,
	\[
	\big|\mathbb{E}_{\tilde p^\star}\mathcal{L}_s(z)-\mathbb{E}_{\tilde p_s}\mathcal{L}_s(z)\big|
	\le C_s  D(\tilde p^\star,\tilde p_s),
	\]
	and the claimed bound \eqref{eq:bound1} follows by convexity of the absolute value and the triangle inequality.

	The three concrete instantiations are immediate and routine from the three cases of Assumption~\ref{assump:loss-regularity}. We present the proof for completeness.

		\textbf{Case A (Wasserstein):}
		Define $\bar{R} = \sum_{s \in S} \lambda_s R(p_s)$ as the weighted average risk. Then:
		\begin{align}
			R(\tilde p^\star) - \bar{R} &= R(\tilde p^\star) - \sum_{s \in S} \lambda_s R_s(\tilde p_s)\\
			&= \sum_{s \in S} \lambda_s [R_s(\tilde p^\star) - R_s(\tilde p_s)]\\
			&= \sum_{s \in S} \lambda_s \left[\mathbb{E}_{z \sim \tilde p^\star}[\mathcal{L}_s(z)] - \mathbb{E}_{z \sim \tilde p_s}[\mathcal{L}_s(z)]\right]
		\end{align}
		
		Since $\mathcal{L}_s$ is $L_s$-Lipschitz, by the Kantorovich-Rubinstein duality:
		\begin{equation}
			\left|\mathbb{E}_{\tilde p^\star}[\mathcal{L}_s(z)] - \mathbb{E}_{\tilde p_s}[\mathcal{L}_s(z)]\right| \leq L_s \cdot W_2(\tilde p^\star, \tilde p_s)
		\end{equation}
		
		Therefore:
		\begin{align}
			\left|R(\tilde p^\star) - \bar{R}\right| &\leq \sum_{s \in S} \lambda_s \left|\mathbb{E}_{\tilde p^\star}[\mathcal{L}_s(z)] - \mathbb{E}_{\tilde p_s}[\mathcal{L}_s(z)]\right|\\
			&\leq \sum_{s \in S} \lambda_s L_s \cdot W_2(\tilde p^\star, \tilde p_s)
		\end{align}
		
		\textbf{Case B (MMD):}
		For $\mathcal{L} \in \mathcal{H}_k$, using the reproducing property:
		\begin{align}
			\left|\mathbb{E}_{\tilde p^\star}[\mathcal{L}_s(z)] - \mathbb{E}_{\tilde p_s}[\mathcal{L}_s(z)]\right| &= \left|\langle \mathcal{L}_s, \mu_{\tilde p^\star} - \mu_{\tilde p_s} \rangle_{\mathcal{H}_k}\right|\\
			&\leq \|\mathcal{L}_s\|_{\mathcal{H}_k} \cdot \|\mu_{\tilde p^\star} - \mu_{\tilde p_s}\|_{\mathcal{H}_k}\\
			&= \|\mathcal{L}_s\|_{\mathcal{H}_k} \cdot \text{MMD}_k(\tilde p^\star, \tilde p_s)
		\end{align}
		
		\textbf{Case C (f-divergence):}
		Using Pinsker's inequality for KL divergence (similar bounds exist for other f-divergences):
		\begin{align}
			\text{TV}(p_s, p^\star) \leq \sqrt{\frac{1}{2}D_{KL}(\tilde p_s \| \tilde p^\star)}
		\end{align}
		
		For bounded $\mathcal{L}_s$:
		\begin{align}
			\left|\mathbb{E}_{\tilde p^\star}[\mathcal{L}_s(z)] - \mathbb{E}_{\tilde p_s}[\mathcal{L}_s(z)]\right| &\leq 2B_s \cdot \text{TV}(\tilde p_s, p^\star)\\
			&\leq B_s\sqrt{2 D_{KL}(\tilde p_s \| \tilde p^\star)}
		\end{align}

\end{proof}

\subsection{Group Disparity Bound (Theorem~\ref{thm:disparity-vanish})}
\label{app:proof-disparity}

\begin{proof}
	For any $s,s'$,
	\begin{align*}
		&\mathbb{E}_{\tilde p_s}\mathcal{L}_s(z)-\mathbb{E}_{\tilde p_{s'}}\mathcal{L}_{s'}(z) \\
		=& \big(\mathbb{E}_{\tilde p_s}\mathcal{L}_s(z)-\mathbb{E}_{\tilde p^\star}\mathcal{L}_s(z)\big)
		+ \big(\mathbb{E}_{\tilde p^\star}\mathcal{L}_s(z)-\mathbb{E}_{\tilde p^\star}\mathcal{L}_{s'}(z)\big)
		\\&+ \big(\mathbb{E}_{\tilde p^\star}\mathcal{L}_{s'}(z)-\mathbb{E}_{\tilde p_{s'}}\mathcal{L}_{s'}(z)\big).
	\end{align*}
	The first and third terms are bounded via (A1) by $C_s D(\tilde p_s,\tilde p^\star)$ and $C_{s'} D(\tilde p_{s'},\tilde p^\star)$. The middle term is bounded by $K$ using (A2). Taking absolute values and further maximizing the middle term over $s,s'$ and the sample path of $\tilde p^\star$ yields \eqref{eq:disparity-time}.
\end{proof}

Now we formalize the condition for the elimination of  $K_{\mathrm{eff}}(t)$ in \eqref{eq:disparity-time}.

In order to eliminate $K_{\mathrm{eff}}(t)$,  we adopt the following \emph{strong} assumption: the fairness-aware barycenter converges to a \emph{single} parameter $z^\dagger$ that equalizes all group losses. This lets us turn the qualitative statement in Theorem~\ref{thm:disparity-vanish} into an \emph{asymptotically vanishing} bound with explicit rates that depend on the chosen discrepancy $D$.

\begin{assumption}[Point--mass strong limit]
	\label{assump:point-mass}
	There exists $z^\dagger \in \Xi$ such that
	\begin{equation}
		\label{eq:equalizer-point}
		\mathcal{L}_s(z^\dagger) = \mathcal{L}_{s'}(z^\dagger) \qquad \forall  s,s' \in S,
	\end{equation}
	and the (empirical) barycenters produced by Fair-BADS satisfy
	\[
	\tilde p_\star^{(t)} \xrightarrow[D]{} \delta_{z^\dagger}
	\quad \text{and} \quad
	D(\tilde p_s^{(t)}, \tilde p_\star^{(t)}) \to 0 \quad \text{for all } s \in S,
	\]
	as $t \to \infty$.
\end{assumption}

We now prove that, under Assumption~\ref{assump:point-mass}, the effective cross--group term in Theorem~\ref{thm:disparity-vanish} \emph{vanishes}. We present the result for the three divergences we consider (Wasserstein, MMD/IPM, and $f$-divergence). The Wasserstein case provides a \emph{sup}-type (support-level) bound; for MMD and $f$-divergences, we obtain clean \emph{expectation}-level bounds.\footnote{One can turn the expectation bounds for MMD/$f$-divergence into support-type statements under additional smoothness/uniform-continuity assumptions; we do not pursue these technicalities here.}

\paragraph{A convenient empirical-particle inequality (Wasserstein).}
When $\tilde p_\star^{(t)}$ is represented by $M$ equally weighted particles $\{z^{(m)}_\star\}_{m=1}^M$,
\begin{equation}
	\label{eq:max-dist-W2}
	\max_{1 \le m \le M} \|z^{(m)}_\star - z^\dagger\|
	 \le 
	\sqrt{M}  W_2\!\big(\tilde p_\star^{(t)}, \delta_{z^\dagger}\big),
\end{equation}
because $W_2^2 = \frac{1}{M} \sum_{m=1}^M \|z^{(m)}_\star - z^\dagger\|^2$ and hence $\max_m \|z^{(m)}_\star - z^\dagger\| \le \sqrt{M} \big(\frac{1}{M} \sum_m \|z^{(m)}_\star - z^\dagger\|^2 \big)^{1/2}$.

\begin{theorem}[Vanishing effective cross--group term]
	\label{thm:Keff-strong}
	Assume \textnormal{(A1)} and Assumption~\ref{assump:point-mass}. Define
	\[
	K_{\mathrm{eff}}(t) 
	:= \sup_{z \in \mathrm{supp}(\tilde p_\star^{(t)})} \max_{s,s' \in S}
	\big|\mathcal{L}_s(z)-\mathcal{L}_{s'}(z)\big|.
	\]
	Then $K_{\mathrm{eff}}(t) \to 0$ as $t \to \infty$. More precisely:
	
	\begin{enumerate}
		\item \textbf{(a) Wasserstein case.} 
		Suppose $D = W_2$ and each $\mathcal{L}_s$ is $L_s$-Lipschitz. Then for every $t$,
		\begin{equation}
			\label{eq:Keff-W2}
			K_{\mathrm{eff}}(t) 
			 \le 
			2 \max_{s \in S} L_s \sqrt{M}  W_2\!\big(\tilde p_\star^{(t)}, \delta_{z^\dagger}\big),
		\end{equation}
		and hence $K_{\mathrm{eff}}(t) \to 0$ as soon as $W_2(\tilde p_\star^{(t)}, \delta_{z^\dagger}) \to 0$.
		
		\item \textbf{(b) MMD/IPM case(expectation level).}
		Suppose $D = \mathrm{MMD}_k$, each $\mathcal{L}_s \in \mathcal{H}_k$ with $\|\mathcal{L}_s\|_{\mathcal{H}_k}\le C_s$. Then
		\begin{equation}
			\label{eq:exp-gap-MMD}
			\max_{s,s' \in S}
			\Big|\E_{\tilde p_\star^{(t)}}[\mathcal{L}_s] - \E_{\tilde p_\star^{(t)}}[\mathcal{L}_{s'}]\Big|
			 \le 
			2 C_{\max}   \mathrm{MMD}_k\!\big(\tilde p_\star^{(t)}, \delta_{z^\dagger}\big)
		\end{equation}
		where $C_{\max} = \max_{s \in S} C_s,$
		and thus the \emph{expected} cross-group gap vanishes as $\mathrm{MMD}_k(\tilde p_\star^{(t)}, \delta_{z^\dagger}) \to 0$.
		
		\item \textbf{(c) $f$-divergence case (expectation level).}
		Suppose $D = D_f$, each $\mathcal{L}_s \in [0,B_s]$, and let $c_f$ be the Pinsker-type constant for $D_f$ (e.g., $c_f=1$ for $\KL$). Then
		\begin{equation}
			\label{eq:exp-gap-fdiv}
			\max_{s,s' \in S}
			\Big|\E_{\tilde p_\star^{(t)}}[\mathcal{L}_s] - \E_{\tilde p_\star^{(t)}}[\mathcal{L}_{s'}]\Big|
			 \le 
			2 \max_{s \in S} B_s \sqrt{2 c_f   D_f\!\big(\tilde p_\star^{(t)}   \big\|   \delta_{z^\dagger}\big)},
		\end{equation}
		and hence the \emph{expected} cross-group gap vanishes whenever $D_f(\tilde p_\star^{(t)} \| \delta_{z^\dagger}) \to 0$.
	\end{enumerate}
\end{theorem}

\subsection{Fairness–Aware SVGD: ELBO Improvement Bound}
\label{app:svgd-proof}

\paragraph{Fairness–aware SVGD improves the ELBO while aligning to $p^\star$.}
We also analyze one SVGD step for a fixed group $s$. Let $q_s^{(\ell)}$ be the empirical particle distribution and consider the transport $t(z)=z+\varepsilon \phi_{\text{fair}}(z)$ with
\[
\phi_{\text{fair}}(z)
=
\mathbb{E}_{z'\sim q_s^{(\ell)}}\Big[k(z',z)\nabla_{z'}\log p_{\text{fair}}(z') + \nabla_{z'}k(z',z)\Big],
\]
\[\log p_{\text{fair}} = \log p_s + \log p^\star.
\]
Let $F(q)=\mathrm{ELBO}(q)= - \mathrm{KL}(q\|p_s) + \text{const}$. Following  standard techniques, a first–order Taylor expansion plus a change of variables argument yields
\[
F(q_s^{(\ell+1)}) - F(q_s^{(\ell)})
= \varepsilon  \mathbb{E}_{q_s^{(\ell)}}\big[\mathrm{tr}\big(\mathcal{A}_{p_s}\phi_{\text{fair}}(z)\big)\big]  +  O(\varepsilon^2),
\]
where $\mathcal{A}_{p}\phi=\nabla \log p^\top \phi + \nabla\cdot \phi$ is the Stein operator. Let $\phi_s^\star$ be the KSD‑optimal direction using $\log p_s$ alone. Decompose
\[
\phi_{\text{fair}} = \phi_s^\star + \Delta, 
\quad
\Delta(z)= \mathbb{E}_{z'\sim q_s^{(\ell)}}\big[k(z',z)\nabla_{z'}\log p^\star(z')\big].
\]
By Cauchy–Schwarz in the RKHS,
\begin{equation}
	\begin{aligned}
		&\big|\mathbb{E}_{q_s^{(\ell)}}[\mathrm{tr}(\mathcal{A}_{p_s}\Delta(z))]\big|\\
		&\le 
		\big\|\nabla \log p_s - \nabla \log p^\star\big\|_{L^2(q_s^{(\ell)})}  \|k\|_{\mathcal{H}} \cdot \mathrm{KSD}(q_s^{(\ell)},p_s),
	\end{aligned}
\end{equation}
which we denote by $C_{\text{KSD}}\cdot \mathrm{KSD}(q_s^{(\ell)},p_s)$. Since $\phi_s^\star$ maximizes the linear functional defining the KSD,
\[
\mathbb{E}_{q_s^{(\ell)}}\big[\mathrm{tr}(\mathcal{A}_{p_s}\phi_s^\star(z))\big]  =  \mathrm{KSD}(q_s^{(\ell)},p_s).
\]
Therefore, for sufficiently small $\varepsilon$ (dropping $O(\varepsilon^2)$) we have 
\[
F(q_s^{(\ell+1)}) - F(q_s^{(\ell)})
\ge
\varepsilon\big(1 - C_{\text{KSD}}\big) \mathrm{KSD}(q_s^{(\ell)},p_s) 
\] where $C_{\text{KSD}}=\|\nabla \log p_s - \nabla \log p^\star\big\|_{L^2(q_s^{(\ell)})}  \|k\|_{\mathcal{H}}$ and the condition $C_{\text{KSD}}<1$  can be ensured in practice by (i) annealing the strength of the barycenter term early on; (ii) adaptively tuning the kernel bandwidth; or (iii) updating $p^\star$ frequently so it stays close to each $p_s$.

%

\begin{remark}
	To explicitly control the interaction between posterior inference and fairness alignment, one can replace
		$\log p_{\text{fair}}(z)
		=
		\log p_s(z) + \log p^\star(z)$ \text{by} $
		\log p_{\text{fair}}^{(\lambda)}(z)
		=
		\log p_s(z) + \lambda  \log p^\star(z),$  which leads to the bound
		\[
		F\!\left(q_s^{(\ell+1)}\right) - F\!\left(q_s^{(\ell)}\right)
		  \gtrsim  
		\epsilon \bigl( 1 - \lambda  C_{\mathrm{KSD}} \bigr) 
		\mathrm{KSD}(q_s^{(\ell)},p_s),
		\]
		for a constant $C_{\mathrm{KSD}}$ that depends on the (squared–)RKHS norm of the kernel and the score
		misalignment $\|\nabla \log \tilde p_s - \nabla \log p^\star\|_{L^2(q_s^{(\ell)})}$.
		Hence, by shrinking $\lambda$ we can always guarantee a strictly positive ascent step (monotone ELBO
		increase) even when the fairness score term is temporarily antagonistic to the likelihood term. When
		the two score fields agree ($C_{\mathrm{KSD}}\!\downarrow 0$), the fairness term no longer slows the ELBO ascent. 
	
\end{remark}

\section{Algorithm}
\label{sec: appendix-algorithm}
We provide the pseudocode of the proposed method in~\cref{alg:fair-bads} to clearly outline the key steps of Fair-BADS. The algorithm maintains group-specific particle approximations of the joint posterior over model parameters $\boldsymbol{\theta}$ and sample weights $\mathbf{w}$. For each group, particles are updated via SVGD, where the update direction is informed by both the group posterior and a shared central distribution that encourages fairness across groups. The central distribution is iteratively computed over group-specific particles and serves as a soft alignment target. At each iteration, particles are updated by combining gradients from the group likelihood and the central prior, resulting in a fairness-aware inference process. This formulation ensures that each group’s learning signal is preserved while ensuring inter-group consistency, which is critical for achieving fair data selection.

\begin{algorithm}[ht]
\caption{Fair-BADS}
\label{alg:fair-bads}
\begin{algorithmic}[1]
\Require Training data $\{(\mathbf{x}_i, y_i, s_i)\}_{i=1}^N$, meta data $\mathcal{D}_m$, group indices $S$, particle count $M$, learning rate $\epsilon$, discrepancy measure $D$.
\State Initialize group-specific particles $\{z_s^{(m)}\}_{m=1}^M$ for each $s \in S$.
\State Initialize global central particles $\{\bar{z}^{(m)}\}_{m=1}^M$.
\For{each training iteration}
    \For{each group $s \in S$}
        \State Extract $\mathcal{D}_t^s \subset \mathcal{D}_t$.
        \For{each particle $z_s^{(m)} = (\boldsymbol \theta_s^{(m)}, \mathbf{w}_s^{(m)})$}
            \State Compute group posterior gradient $\nabla_z \log p_s(z)\big|_{z = z_s^{(m)}}$ from \cref{eq: -logp}.
        \EndFor
        \State Estimate gradient $\nabla_z \log p^\star(z)\big|_{z = z_s^{(m)}}$ from central distribution.
        \State Combine gradients: $\nabla_z \log p_\text{fair}(z) = \nabla_z \log \Tilde{p}_s(z) + \nabla_z \log \Tilde{p}^\star(z)$.
        \State Update $\{z_s^{(m)}\}_{m=1}^M$ via SVGD. 
    \EndFor
    \State Compute central distribution $\bar{\mathbf{Z}}$ across all groups as in \cref{eq: updated_average_z}.
    \State Update KDE reference using central particles.
\EndFor
\end{algorithmic}
\end{algorithm}

\end{document}